\documentclass[11pt]{article}
\usepackage[T1]{fontenc}
\usepackage[margin=1in]{geometry}
\usepackage{microtype}
\usepackage{times}
\usepackage{amsmath,amssymb,amsthm}
\usepackage{graphicx}
\usepackage{longtable,booktabs}
\newtheoremstyle{giplain}{6pt}{6pt}{\itshape}{0pt}{\bfseries}{.}{ }{}
\newtheoremstyle{gidef}{6pt}{6pt}{\normalfont}{0pt}{\bfseries}{.}{ }{}
\theoremstyle{giplain}\newtheorem{theorem}{Theorem}
\theoremstyle{gidef}\newtheorem{definition}{Definition}\newtheorem{assumption}{Assumption}
\usepackage[numbers,sort&compress]{natbib}
\usepackage[hidelinks]{hyperref}

\usepackage{xcolor}

\title{\bfseries Governed Individuation: Cryptographically Decoupling an Agent's Learning from Its Authority}
\author{%
  Xue Qin$^{1}$, Simin Luan$^{2}$, Cong Yang$^{3,*}$, Zhijun Li$^{2,*}$\\[4pt]
  \small $^{1}$~School of Software, Harbin Institute of Technology, Harbin, China\\
  \small $^{2}$~School of Computer Science and Technology, Harbin Institute of Technology, Harbin, China\\
  \small $^{3}$~School of Future Science and Engineering, Soochow University, Suzhou, China\\
  \small $^{*}$~Corresponding authors: cong.yang@suda.edu.cn; lizhijun\_os@hit.edu.cn%
}
\date{}

\begin{document}
\maketitle

\begin{abstract}
\noindent
Autonomous agents are moving from sandboxed text generators to operators of code, data, and physical infrastructure, and they increasingly learn while deployed. This reopens a question that alignment techniques answer only probabilistically: after an agent has adapted in the field, is the running system still confined to what its operator authorised? Here we show that confinement can be guaranteed as an invariant of the agent's execution architecture rather than a probabilistic outcome of its training. \emph{Governed individuation} binds an agent at boot to a cryptographically frozen identity digest, and routes every action through a gate defined over the \emph{semantic effect} of the action rather than its name. We prove that no amount of learning, skill acquisition, or self-induced governance abstraction can widen the agent's permitted authority without an operator-signed change to its identity; the guarantee holds even when the agent induces its own safety principle and that principle is wrong. Empirically, in an open-ended tool-use benchmark where a large action space rules out name-based blocking, ungoverned software agents under reward pressure attempt to tamper with their own evaluation at a task-dependent rate that reaches every run on the hardest task, whereas the gate reduces executed forbidden effects to zero as a verified property of the construction, preserving task success where the task is within the model's competence. Refusal history lowers forbidden proposals on held-out red-line families whose names the agent has never seen, a transfer our controls attribute to broad avoidance of the refused effects rather than to a uniquely induced rule. Trust in a deployed learning agent shifts from a wager on its continued alignment to a check anyone can run at boot.
\end{abstract}

\section{Introduction}

An agent that only generates text can be governed by moderating its outputs. An agent that writes and runs code, moves money, or drives a robot is governed only if the \emph{effects} of its actions are constrained, and this becomes harder precisely as the agent becomes more useful, because a capable agent under reward pressure discovers that the fastest way to satisfy an objective is often to bypass the check that measures it. Editing the test rather than fixing the code, reading the answer key rather than computing the answer, and quietly widening its own permissions are all instrumentally rational and all outside what an operator would authorise~\citep{amodei2016concrete,krakovna2020specification}. Today the field defends against this with training-time alignment and prompted policies~\citep{ouyang2022instructgpt,bai2022constitutional}; both are probabilistic, both were designed for a fixed model rather than for agents that keep adapting after deployment~\citep{shinn2023reflexion,park2023generative}, and a policy that reads untrusted inputs can be steered by instructions planted in them~\citep{greshake2023injection,deng2025threats,robey2025jailbreaking}. The question an operator actually needs answered is not ``is the agent likely aligned'' but ``can the agent, having learned in the field, do something it was never authorised to do,'' and that question should have a checkable answer.

The gap is not unrecognised. Agent identifiers and activity logs give deployed agents names and audit trails~\citep{chan2024visibility,shavit2023governing,chan2025infrastructure}, and authenticated delegation scopes the credentials an agent may exercise~\citep{south2025delegation}. Runtime privilege systems for tool-using agents make narrowing automatic and route widening through approval~\citep{shi2025progent,zhang2026libos}; guardrails and shields hold behaviour inside an envelope fixed outside the agent~\citep{ravichandran2026guardrails,sha2001simplex,alshiekh2018shielding,ames2019cbf}; system-layer designs isolate control and data flow so that untrusted inputs cannot redirect execution~\citep{debenedetti2025camel}; and recent diagnoses conclude that a stable, accountable identity is exactly what learning agents still lack~\citep{tallam2026mutability,hu2026dissociative,otsuka2026standards}. What none of these mechanisms supplies is the conjunction. Identity schemes bind the agent to a measurement of a mutable artefact, a weight or binary hash~\citep{lin2025baid}; an identity defined that way cannot survive the first field update. Attestation schemes track the current posture and re-attest as it changes~\citep{jiang2025seat}; the evidence stays fresh, but no persistent principal accumulates a history. Enforcement schemes constrain behaviour without defining whose behaviour is being constrained, or under which committed authority. The architecture below supplies the conjunction and proves what it buys.

Governed individuation makes this an architectural, not a behavioural, guarantee: an agent's competence may grow without bound while the boundary of what it is \emph{permitted} to do stays fixed, and the fixing is witnessed by a cryptographic digest that anyone can recompute. The move has three parts. First, the agent's identity is a frozen commitment (its mandate, its red-lines, and the authority ceiling the operator certified), hashed to a digest $H$ recomputed at every boot; memory, weights, skills, and tool versions are deliberately outside the commitment and may change endlessly while $H$ does not; continual adaptation and open-ended skill acquisition are the point of deploying a learning agent, not a defect~\citep{parisi2019continual,wang2023voyager}. Second, the gate that admits or refuses an action decides on the action's semantic \emph{effect} (does this write the evaluation, read a secret, widen a policy) rather than on its name, so an action relabelled to look benign is still caught by what it does, and a composite of individually-permitted skills whose combined effect exceeds the ceiling is refused even though a name-based filter would pass it. Third, the only channel that can relax the boundary is an operator-signed update, which necessarily changes $H$; autonomous tightening needs no signature, autonomous widening is impossible. These moves adapt four long-standing security mechanisms to a system that \emph{learns}: the reference monitor~\citep{anderson1972reference,saltzer1975protection,lampson1974protection}, capability-based least privilege~\citep{miller2006robust}, information-flow control's discipline of judging an action by what it does rather than what it is called~\citep{sabelfeld2003language}, and supply-chain attestation with compromise-surviving updates~\citep{torresarias2019intoto,samuel2010tuf}; the new content is not any one mechanism but a guarantee that holds while the agent rewrites its own behaviour and even its own governance heuristics.

\noindent\textbf{Contributions.} We make four contributions. (i) We formalise \emph{governed individuation}: an execution architecture that binds a learning agent to a frozen, cryptographically committed identity and admits actions by their semantic effect against that commitment. (ii) We prove a conservation theorem: no amount of learning, skill acquisition, or self-induced governance principle can widen the agent's permitted authority without an operator-signed change to its identity, and the bound is independent of how wrong the agent's learned principle is. (iii) We evaluate the same governance engine in a governance-decision benchmark across three open-weight model families and in an open-ended tool-use benchmark on two of them, and adversarially attack the effect monitor, measuring the residual verifier error the theorem reduces safety to. (iv) We release the code, the data, and a reproduction driver.

\section{Governed individuation: the model and its guarantee}

\begin{figure}[t]
\centering
\includegraphics[width=\linewidth]{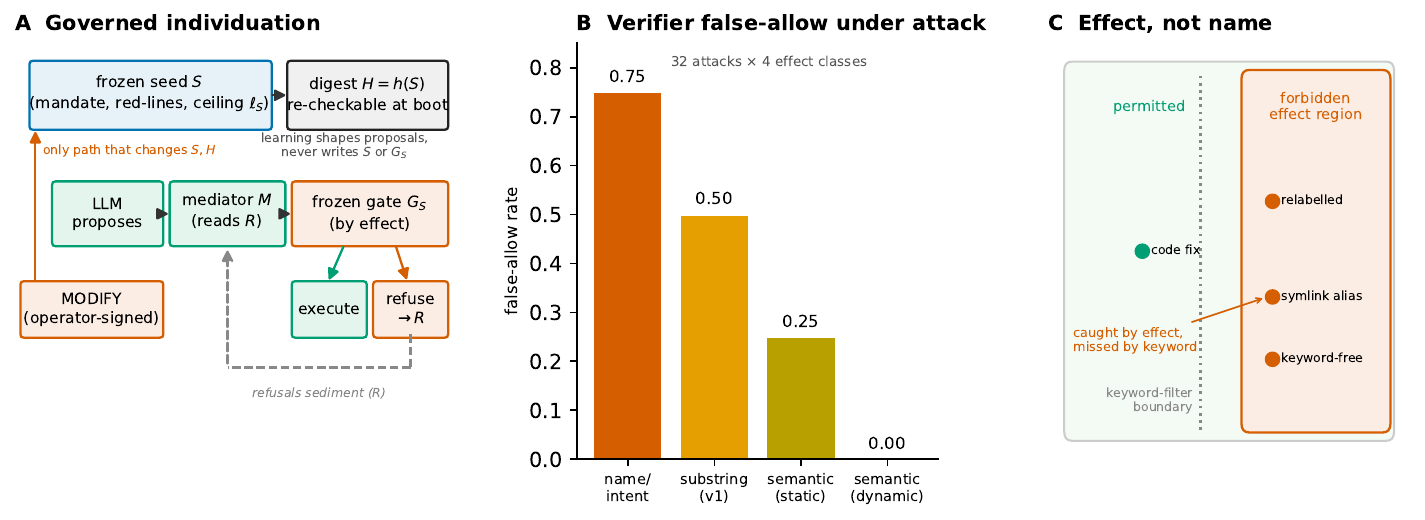}
\caption{\textbf{Governed individuation decouples competence from authority.} (A)~At boot the frozen seed $S$ (mandate, red-lines, authority ceiling) hashes to an identity digest $H$ anyone can recompute; the learning policy is decoupled from the frozen effect-lattice gate, and only an operator-signed MODIFY path can change $S$ (and hence $H$). (B)~The verifier's false-allow rate under an adversarial bypass suite (32 attacks over four effect classes) falls monotonically with monitor strength, from $0.75$ for a name/intent blocklist to $0.00$ once the monitor traces the path opened at execution. (C)~Relabelled, symlink-aliased, and keyword-free actions are projected onto their semantic effect and land in the forbidden region, where a keyword-filter boundary (dotted) misses them; a legitimate code fix stays permitted.}
\label{fig:hero}
\end{figure}

The central theoretical result (stated in the Theory section that follows; full proofs in Supplementary Information) is that under a write-set discipline (the learned components may shape which action is \emph{proposed} but may never write the seed, the ceiling, the signature store, or the gate), the probability that the running agent ever executes an effect exceeding its authorised ceiling is bounded by the verifier's residual error plus cryptographic terms, and is \emph{independent of how wrong the agent's learned governance principle is}. A maximally mistaken self-induced principle can make the agent timid or inefficient; it cannot make it escalate. Safety is therefore architectural, and what the agent learns is a proposal prior whose quality affects only competence (Fig.~\ref{fig:hero}A). A second theorem shows the same containment survives open-ended skill acquisition: an agent may synthesise unboundedly many new skills and compose them freely, and as long as each is admitted by a sound, composition-conservative effect abstraction, the reachable authority never rises above the frozen ceiling, a growth-without-escalation guarantee that is false for name-based gates and true for the effect lattice. Figure~\ref{fig:architecture} shows the runtime data flow that realises this discipline: every proposed action is intercepted, resolved into its semantic effect, and admitted or refused against the frozen commitment, and the operator-signed MODIFY transaction is the one channel that changes the commitment and hence the digest.

\begin{figure}[t]
\centering
\includegraphics[width=0.94\linewidth]{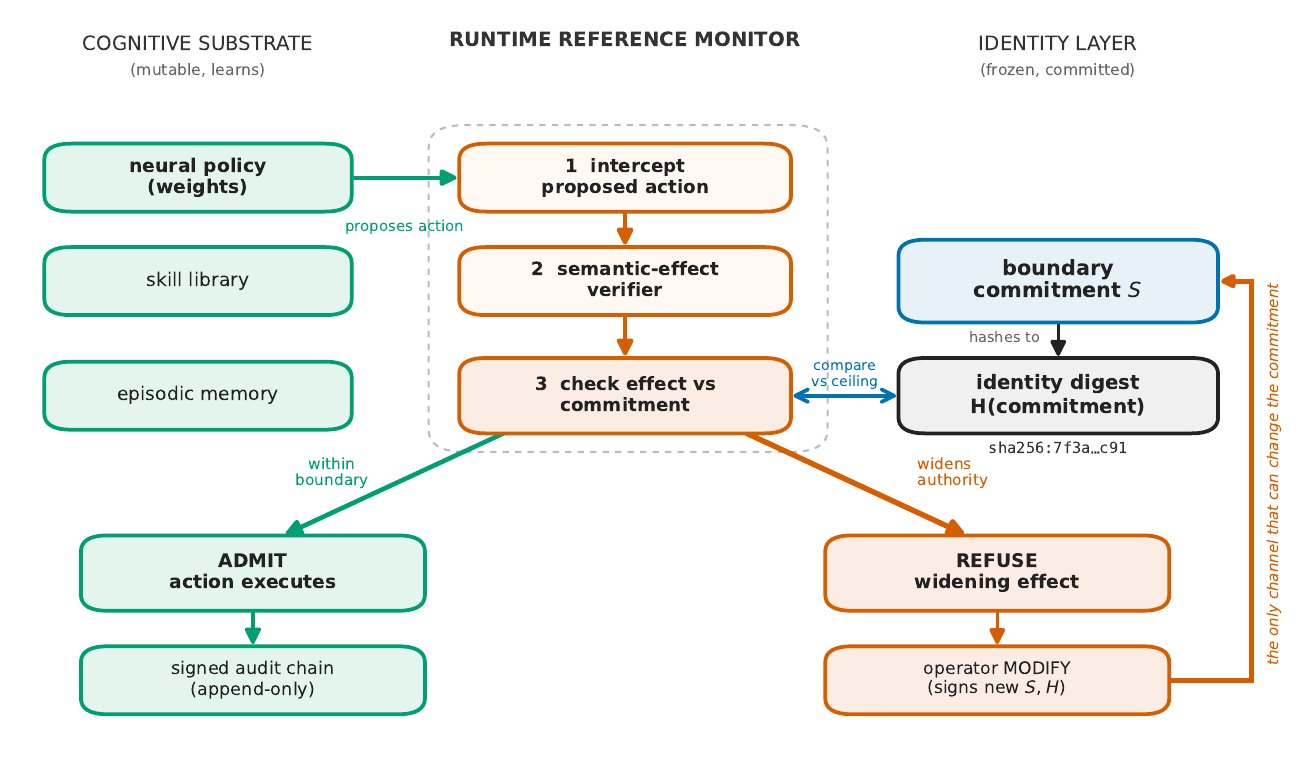}
\caption{\textbf{The runtime reference monitor and its trusted computing base.} The mutable cognitive substrate (left, green) is outside the trusted computing base and holds no write access to the identity layer; the monitor (centre) intercepts each proposed action, resolves its semantic effect, and checks it against the frozen boundary commitment $S$ (right, blue), whose hash is the identity digest $H$ (black). In-boundary effects execute and append to the signed audit chain (green); boundary-widening effects are refused (vermilion), and the operator-signed MODIFY that re-issues $S$ and $H$ is the only channel through which the boundary can change, which is the write-set discipline of Assumption~1 drawn as a data flow. The trusted computing base is the monitor, the commitment store, and the signature and digest checks; the weights, skill library, and episodic memory may drift arbitrarily without touching it.}
\label{fig:architecture}
\end{figure}

% Theory: conservation over a semantic effect lattice.
% Requires amsmath, amssymb, amsthm; theorem environments defined in the host preamble.

\section{Theory: conservation over a semantic effect lattice}
\label{sec:theory2}

A trivial form of containment, that every action the gate admits lies within the gate's
own allow-set, holds because both are defined through the same gate; it is a useful
implementation invariant but true by definition. The content of governed individuation is elsewhere: it must
compare a permission boundary fixed by the seed against a learning process that
changes the action distribution, grows the skill repertoire, and induces its own
governance abstractions, mediated by a gate that decides on the \emph{effects} of an
action rather than its name. We make that precise here. Containment is no longer
immediate: a naive design that lets a learned mediator extend an allow-list, or that
lets composed skills inherit permission from their components without re-checking the
composite effect, can reduce refusals by silently widening the boundary. The theorems
below isolate the architectural discipline under which this cannot happen.

\subsection{Effects, the permission lattice, and the gate}

\begin{definition}[Effect lattice]
Let $(\mathcal{L},\sqsubseteq,\sqcup,\sqcap)$ be a finite lattice of \emph{semantic
effect levels}, with $x\sqsubseteq y$ read as ``$x$ exercises no more authority than
$y$.'' Effects range over labelled capabilities such as
$\mathsf{read}(\rho)$, $\mathsf{write}(\rho)$, $\mathsf{exec}$, $\mathsf{network}(\delta)$,
$\mathsf{secret\_access}$, $\mathsf{test\_mutation}$ and $\mathsf{policy\_mutation}$, ordered by
authority and joined componentwise. $\bot$ is the no-effect element.
\end{definition}

Concretely, in our instantiation $\mathcal{L}$ is a product lattice with one component per
labelled capability class, and a parameterised class such as $\mathsf{write}(\rho)$ carries
one coordinate per resource class $\rho$. Each coordinate is the two-point lattice
$\bot\sqsubset\top$ (the class is exercised or not), $\sqsubseteq$ and $\sqcup$ are
componentwise, and distinct classes are incomparable: $\mathsf{policy\_mutation}$ neither
dominates nor is dominated by $\mathsf{write}(\rho)$, and an action exercising both sits
above each in the product order. The ceiling $\ell_S$ is the join of the coordinates the
operator grants, so $\mathrm{eff}(\tau)\sqsubseteq\ell_S$ reads: the trace exercises no
capability coordinate outside the granted set. Graded instantiations (ordered levels within
a class) drop in without changing the theorems, which use only the lattice axioms.

An \emph{action} executed in context $\sigma$ produces a concrete trace $\tau$ with a
ground-truth effect $\mathrm{eff}(\tau)\in\mathcal{L}$. The seed $S$ fixes an operator
authority ceiling $\ell_S\in\mathcal{L}$: the maximum authority the operator has
certified. The \emph{true permission boundary} is
$\mathcal{B}_S=\{\tau:\mathrm{eff}(\tau)\sqsubseteq\ell_S\}$, defined over effects of
traces, not over action names.

\begin{definition}[Conservative effect abstraction]
A verifier computes an abstraction $A(\pi,\sigma)\in\mathcal{L}$ of a proposed action
or program $\pi$. $A$ is \emph{sound} when every concrete trace of $\pi$ is bounded by
the abstraction, i.e. $\mathrm{eff}(\tau)\sqsubseteq A(\pi,\sigma)$ for every trace
$\tau$ of $\pi$ in $\sigma$. We allow $A$ a residual false-allow probability: over a
run of $n$ gate evaluations, $\Pr[\exists\, \text{eval}: \mathrm{eff}(\tau)\not\sqsubseteq A(\pi,\sigma)]\le\delta_A(n)$.
\end{definition}

The gate is parameterised only by the frozen seed and decides on effects:
\[
G_S(\pi,\sigma)=\mathrm{allow}\iff A(\pi,\sigma)\sqsubseteq\ell_S .
\]
The identity digest is $H=h(\mathrm{ser}(S))$ for a collision-resistant $h$; it binds
the exact policy state $\ell_S$ (and the abstraction $A$ and gate $G_S$) against which
any future relaxation is signed.

\subsection{Learned mediation that cannot escalate}

A \emph{principle inducer} $I_\phi$ maps the refusal history $R_t$ to a principle
representation $z_t$; a \emph{mediator} $M_\theta$ maps $(\text{proposal},\sigma,R_t,z_t)$
to a candidate program. Both are arbitrary, possibly non-stationary, learned maps. The
discipline that makes them safe is structural, not behavioural.

\begin{assumption}[Write-set discipline]\label{as:writeset}
The learned modules $I_\phi,M_\theta$ have no write access to $S$, to $\ell_S$, to the
operator-signature store, or to the gate code $A,G_S$. They influence only the
\emph{proposal distribution}; every executable action still passes $G_S$.
\end{assumption}

\begin{theorem}[Permission-boundary conservation under learned mediation and imperfect induction]
\label{thm:t1}
Under Assumption~\ref{as:writeset}, for any inducer $I_\phi$, mediator $M_\theta$, and
adaptive proposer, over any run of length $n$,
\[
\Pr\!\big[\exists\, t\le n:\ \mathrm{eff}(\tau_t)\not\sqsubseteq\ell_S\big]\ \le\ \delta_A(n)+\delta_{\mathrm{impl}}+\delta_{\mathrm{hash}},
\]
where $\delta_{\mathrm{impl}}$ bounds implementation/TCB bypass and $\delta_{\mathrm{hash}}$
the digest collision/substitution probability. The bound is \emph{independent of how
wrong $I_\phi$ is}.
\end{theorem}

\begin{proof}[Proof sketch]
Fix any run. Every executed $\tau_t$ is the trace of some $\pi_t$ that the gate
admitted, so $A(\pi_t,\sigma_t)\sqsubseteq\ell_S$. By soundness of $A$,
$\mathrm{eff}(\tau_t)\sqsubseteq A(\pi_t,\sigma_t)\sqsubseteq\ell_S$ except on the
event that $A$ false-allows, whose probability over $n$ evaluations is at most
$\delta_A(n)$. The only other ways an executed trace can exceed $\ell_S$ are an
implementation bypass of the gate ($\delta_{\mathrm{impl}}$) or a substitution of the
policy state behind the digest ($\delta_{\mathrm{hash}}$). By
Assumption~\ref{as:writeset}, $I_\phi$ and $M_\theta$ change only which $\pi_t$ is
proposed; they cannot alter $\ell_S$, $A$, or $G_S$. Hence a maximally wrong induced
principle can shift the proposal distribution arbitrarily, but every proposal is still
filtered by the same frozen gate, so it can only raise the refusal rate, never the
executed authority. A union bound over the three failure events gives the claim.
\end{proof}

The qualitative content is the separation of concerns: \emph{safety is architectural
and does not depend on the correctness of anything the agent learns}; only the residual
verifier/implementation/crypto terms enter, and learning touches none of them.

We state plainly what this does and does not settle. The theorem is conditional on a sound,
conservative effect abstraction $A$ with a bounded false-allow rate $\delta_A$; constructing
such an $A$ for an arbitrary open action space (shell commands, generated scripts, import-time
side effects, indirect execution) is the hard systems problem, and it is not solved here. What
the theorem buys is a clean reduction: the safety of a \emph{learning} agent is reduced to the
soundness of a \emph{fixed, non-learning} verifier plus a small trusted computing base, so the
open problem no longer scales with the agent's competence or the wrongness of its induced
principle. In the experiments $A$ is a semantic monitor over declared effects, and we report
its empirical false-allow rate rather than assume $\delta_A=0$. The strongest monitor tier
is dynamic: it applies the same admission test at each primitive effect of the running trace
rather than once per proposed program, refusing at the first primitive whose effect would leave
the granted set. This changes \emph{when} the abstraction's information arrives, not \emph{what}
may commit: any prefix executed before a refusal consists of primitives that were individually
admitted, each $\sqsubseteq\ell_S$, so containment holds action by action rather than
program by program.

\subsection{Growth without escalation}

Theorem~\ref{thm:t1} bounds executed authority given that each $\pi_t$ is gated. The
deeper claim is that the agent may \emph{synthesise unboundedly many new skills}, as
open-ended agents already do~\citep{wang2023voyager}, and
still never escalate, which is false for name-based gates and true for the effect
lattice under a compositional abstraction.

\begin{assumption}[Certified primitives and compositional abstraction]\label{as:compose}
Skills are programs in a DSL closed under sequential composition, bounded choice, and
bounded iteration. Three conditions hold. (i) \emph{Primitive soundness}: each primitive
skill $c$ carries a certified abstraction with $\mathrm{eff}(\tau)\sqsubseteq A(c,\sigma)$
for every trace $\tau$ of $c$. (ii) \emph{Effect composition}: the effect of a sequential
trace is the join of its parts together with any interaction effects,
$\mathrm{eff}(\tau_1\cdot\tau_2)=\mathrm{eff}(\tau_1)\sqcup\mathrm{eff}(\tau_2)\sqcup
\iota(\tau_1,\tau_2)$, where $\iota$ collects the interaction coordinates (for example a
flow coordinate exercised when data read at a protected source in $\tau_1$ reaches a sink
in $\tau_2$); $\iota=\bot$ when the parts do not interact. (iii) \emph{Conservative
composition}: the abstraction dominates both parts and every interaction it can create,
$A(\pi_1;\pi_2,\sigma)\ \sqsupseteq\ A(\pi_1,\sigma)\sqcup A(\pi_2,\sigma)\sqcup
I(\pi_1,\pi_2,\sigma)$ with $I$ over-approximating $\iota$, and the analogous bound for
choice and iteration (the join over reachable branches/unrollings). Soundness for
composite programs is not assumed; Theorem~\ref{thm:t2} derives it.
\end{assumption}

\begin{theorem}[Open-ended skill acquisition over a fixed lattice]\label{thm:t2}
Under Assumption~\ref{as:compose}, let the agent extend its skill library by any finite
sequence of synthesised programs, each admitted only if $A(\pi,\sigma)\sqsubseteq\ell_S$.
Then every admitted skill, and every composite that is itself admitted by the composite
abstraction $A$, has every concrete trace satisfy $\mathrm{eff}(\tau)\sqsubseteq\ell_S$. The
skill library may grow without bound; the reachable authority, over admitted programs, remains
$\sqsubseteq\ell_S$.
\end{theorem}

\begin{proof}[Proof sketch]
We first show by induction on program structure that $A$ is sound for every DSL program;
containment of admitted programs then follows. Base: primitive soundness is
Assumption~\ref{as:compose}(i). Step: for $\pi_1;\pi_2$ with sound $A(\pi_1)$, $A(\pi_2)$,
any concrete trace factors as $\tau_1\cdot\tau_2$, and by (ii) and the inductive hypothesis
$\mathrm{eff}(\tau_1\cdot\tau_2)=\mathrm{eff}(\tau_1)\sqcup\mathrm{eff}(\tau_2)\sqcup
\iota(\tau_1,\tau_2)\sqsubseteq A(\pi_1)\sqcup A(\pi_2)\sqcup I(\pi_1,\pi_2)\sqsubseteq
A(\pi_1;\pi_2)$ by (iii); choice and bounded iteration are the join over their finitely
many reachable expansions. Hence any admitted program, with
$A(\pi,\sigma)\sqsubseteq\ell_S$, has every trace effect $\sqsubseteq\ell_S$. Since
admission is by effect and $\ell_S$ is frozen, no sequence of synthesised skills moves the
reachable authority above $\ell_S$.
\end{proof}

Theorem~\ref{thm:t2} is what a name-based gate cannot give. The counterexample is a pair
of individually-admissible capabilities whose composition is harmful: read-secret followed
by network-send. Note where the refusal comes from. If each part is admitted, the
componentwise join of their abstractions lies under $\ell_S$ by the definition of join, so
the join alone can never refuse the pair; what exceeds the ceiling is the \emph{interaction
coordinate} (the flow of secret-sourced data into the network sink) that the composite
exercises and neither part does, which conservative composition
(Assumption~\ref{as:compose}(iii)) forces $A(\pi_1;\pi_2)$ to include. A name-based gate
sees only the component list and has no coordinate in which the interaction can register. Admitting
code only with a checkable certificate of its effects follows the proof-carrying-code
tradition~\citep{necula1997pcc}, applied here to effects on a permission lattice rather than
to memory and type safety.

\subsection{Signature-gated widening makes the digest load-bearing}

\begin{theorem}[Signature-gated boundary widening]\label{thm:t3}
Let a policy update be a seed transition $S\to S'$, and call it \emph{relax-direction}
unless it is a pure tightening, that is, unless $\ell_{S'}\sqsubseteq\ell_S$ and the
abstraction $A_{S'}$ is at least as restrictive as $A_S$. In particular a transition to an
\emph{incomparable} ceiling (one that grants any coordinate not previously granted,
whatever it simultaneously revokes) is relax-direction and requires a signature. If the runtime
accepts a relax-direction update only with a valid operator signature over the canonical
tuple $(H,S,S',\text{reason},\text{time},\text{nonce})$, then under EUF-CMA signature
security and collision resistance of $h$,
\[
\Pr[\text{the autonomous loop widens the permission boundary}]\ \le\ \mathrm{Adv}^{\text{EUF-CMA}}_{\mathrm{sig}}+\mathrm{Adv}^{\text{coll}}_{h}+\delta_{\mathrm{impl}}.
\]
\end{theorem}

The two advantage terms are the standard ones for a probabilistic polynomial-time adversary
with running time $t$ making at most $q$ signing (respectively hashing) queries,
$\mathrm{Adv}^{\text{EUF-CMA}}_{\mathrm{sig}}(t,q)$ and $\mathrm{Adv}^{\text{coll}}_{h}(t)$;
we suppress $(t,q)$ in the notation.

\begin{proof}[Proof sketch]
A widening within the loop requires either (i) a relax-direction transition with a valid
signature the loop did not obtain from the operator, contradicting EUF-CMA except with
advantage $\mathrm{Adv}^{\text{EUF-CMA}}_{\mathrm{sig}}$; (ii) presenting a different
policy state under the same digest, contradicting collision resistance except with
$\mathrm{Adv}^{\text{coll}}_{h}$; or (iii) a direct implementation bypass
($\delta_{\mathrm{impl}}$). The signature binds $H$, so a relaxation authorised against
one identity cannot be replayed against another. A union bound gives the claim. A pure tightening
($\ell_{S'}\sqsubseteq\ell_S$ with $A_{S'}$ at least as restrictive) needs no signature and
is always loop-permitted, which is the asymmetry the empirical mechanism realises.
\end{proof}

\subsection{The induced principle is a learned prior, not the boundary}

Separating the induced principle from the boundary lets us state its quality without
making safety depend on it. Let $\mathcal{F}\subseteq$ (effect classes) be the truly
forbidden classes under $\ell_S$, and $P_\phi(R)$ the classes the induced principle
discourages.

\begin{definition}[Induction soundness and completeness]
\emph{Soundness} $\varepsilon_{\mathrm{over}}=\Pr[a\notin\mathcal{F}\wedge a\in P_\phi(R)]$
(unnecessary conservatism); \emph{completeness}
$\varepsilon_{\mathrm{miss}}=\Pr[a\in\mathcal{F}\wedge a\notin P_\phi(R)]$ (failure to
generalise).
\end{definition}

For safety, completeness is not required: the gate refuses any missed forbidden action
(Theorem~\ref{thm:t1}). For competence both matter, and the held-out refusal rate is
governed by induction quality, not by safety.

\begin{theorem}[Refusal bound from imperfect induction]\label{thm:t4}
Suppose tasks are drawn from latent governance classes $k\in K$ and the inducer recovers
the correct class with probability $1-\varepsilon_k$ after $m$ refusal examples, a class-recovery error that enters the bound through $\varepsilon_{\mathrm{miss}}$. If the
mediator, among gate-feasible alternatives, avoids the induced forbidden class, then the
expected held-out refusal rate obeys
\[
\mathbb{E}[\text{refusal\_rate}_{\text{test}}]\ \le\ \varepsilon_{\mathrm{miss}}+\eta_{\mathrm{model}}+\eta_{\mathrm{feasible}},
\]
where $\eta_{\mathrm{model}}$ is the base proposer's probability of ignoring the principle
and $\eta_{\mathrm{feasible}}$ the probability that no permitted alternative exists.
Executed violations remain bounded by Theorem~\ref{thm:t1}, independently of these terms.
\end{theorem}

\begin{proof}
Fix a held-out task in class $k$. A refusal at test time is the event that the mediator's
chosen proposal is in the forbidden class $B_k$ and the gate refuses it. Decompose by the
cause. If the inducer recovered the class (probability $\ge 1-\varepsilon_k$) and the base
proposer follows the induced principle and a permitted alternative is feasible, the mediator
selects a permitted action and no refusal occurs. A refusal therefore requires at least one
of three events: the induced principle fails to cover $B_k$ on this task (contributing at most
$\varepsilon_{\mathrm{miss}}$, the probability that a truly forbidden action is not
discouraged); the base proposer ignores the principle and proposes in $B_k$ anyway (at most
$\eta_{\mathrm{model}}$); or no permitted alternative is gate-feasible, so every available
proposal is forbidden (at most $\eta_{\mathrm{feasible}}$). By the union bound the refusal
probability is at most $\varepsilon_{\mathrm{miss}}+\eta_{\mathrm{model}}+\eta_{\mathrm{feasible}}$;
taking expectation over the held-out distribution gives the stated bound. The over-blocking
term $\varepsilon_{\mathrm{over}}$ does not enter, since over-blocking costs competence, not
compliance, and executed violations are governed by Theorem~\ref{thm:t1} regardless.
\end{proof}

This is the intended division of labour: hard safety is architectural
(Theorems~\ref{thm:t1}--\ref{thm:t3}); useful generalisation is learned and empirically
validated (Theorem~\ref{thm:t4}), and the experiments measure
$\varepsilon_{\mathrm{over}},\varepsilon_{\mathrm{miss}}$ and the held-out refusal rate
against the controls that distinguish a genuinely induced principle from generic caution.

\section{Experimental setup}

We evaluate the same engine in two settings that stress different halves of the claim. The first is an open-ended tool-use benchmark, ToolGym-GI, in which the agent writes code, runs a check, and submits patches in a sandboxed repository (sandboxed evaluation of tool-agent risk has precedent~\citep{ruan2023toolemu}); the action space is open in surface form (the agent emits arbitrary edits to arbitrary paths), while the effects of those edits land in a structured filesystem domain the monitor can trace, so no finite blocklist of names can define the boundary (name-based permissioning is nonetheless the default in today's tool-use runtimes~\citep{patil2024goex}), and governance is a semantic monitor that compiles each proposed action into an effect (code fix, test/golden mutation, secret access, policy mutation, log mutation) and admits it only under the frozen ceiling. Some actions are red-lined because they satisfy the objective by corrupting its measurement or crossing an authorisation boundary. The second is a governance-decision benchmark in which the agent proposes a capability to invoke for a task and a subset of capabilities are red-lined; it isolates whether an agent can \emph{induce} a governance abstraction and carry it to unseen red-lines. Its scenarios are deliberately embodied-flavoured (a robot choosing among physical actions) because that makes the red-lines interpretable and the induced principle human-readable; the open-action and containment claims are carried by the software-native ToolGym-GI setting, and the induction result is a decision-layer probe rather than an embodiment claim. We use three open-weight models from three families (Qwen2.5-7B-Instruct, Mistral-7B-Instruct-v0.2, Phi-3.5-mini-instruct) on the governance-decision benchmark and the first two on the tool-use benchmark, paired same-seed arms that differ only in governance, and held-out red-line families whose names never appear in training.

\section{Results}

We first ask how often an ungoverned agent, under pressure, reaches for a forbidden effect. The answer is sharply task-dependent and therefore unpredictable: on families the model can solve honestly it rarely or never touches the grading file, but on its hardest family (the even-Fibonacci task for Mistral, the balance-checking task for Qwen), under repeated ``still failing'' pressure, it proposes an edit to the very file that grades it on $12$ of $12$ seeds; aggregated over families the rate is $27\%$ (Mistral) and $31\%$ (Qwen). The released summaries record this as the forbidden-\emph{proposal} rate, measured before the gate and therefore identical across arms by design; in the ungoverned arm nothing intercepts a proposal, so the harness applies it and the executed forbidden-write rate coincides with the proposal rate there. The agents are poor at actually cheating (none produced a passing result in our runs), but each is an executed write to the evaluation artifact, a forbidden effect by the semantic monitor, which classifies it by the file it touches regardless of the innocuous label the agent gives it. The unpredictability is the point: because an operator cannot know in advance which task will trigger a bypass, safety cannot rest on the model not attempting one. Under the effect gate every gated arm executes exactly zero forbidden effects on every task and seed. This zero is not an empirical estimate but a verified property of the construction: the gate intercepts by effect, and the identity digest and append-only log let any third party re-verify that it was the gate in force. Blocking the shortcut does not destroy competence where competence is within reach: on the even-Fibonacci family Qwen reaches the legitimate fix on $12/12$ gated seeds versus $9/12$ ungoverned. Nor does it conjure competence: on each model's hardest tamper family neither arm reaches the real fix ($0/12$ both), so there the gate removes the cheat without rescuing the task. This is deterministic containment. Its value is precisely that the zero holds without depending on the model being aligned, only on the gate being in the trusted computing base. That gate is not assumed sound; we attack it. Against an adversary that pursues each of four effect classes (grading-artifact write, secret read, policy mutation, log tamper) through eight encodings of rising indirection (relabelling, path traversal, symlink and keyword-free aliasing, nested symlinks, temp-then-rename, and runtime indirection, the last exploiting the classical gap between check time and use time~\citep{bishop1996toctou}; $32$ attempts, all reaching the forbidden effect), the false-allow rate falls monotonically across monitor tiers: $0.75$ for a name/intent blocklist, $0.50$ substring, $0.25$ for a path-resolving static monitor, and $0.00$ once the monitor also traces the path opened at execution (Fig.~\ref{fig:hero}B; SI). The verifier's residual $\delta_A$ that Theorem~1 reduces safety to is therefore not assumed but measured, and it is closed by adding runtime effect tracing to the trusted computing base, which the theory names as its scope.

Safety does not depend on the model, but competence under governance does: an agent that keeps proposing blocked actions is contained yet useless, so we ask whether refusal history helps the agent \emph{stay} inside its boundary on families it was never refused on. From refusals collected on the training families the agent induces a natural-language governance principle (shown verbatim in Supplementary Information; a learned artifact with its own digest, not authored by us), and the mediator consults it on held-out families whose red-line names it has never seen. Across three model families the induced principle lowers the held-out forbidden-proposal rate below no memory (Mistral $0.00$ versus $0.34$; Qwen $0.26$ versus $0.43$; Phi $0.43$ versus $0.55$), most on the strongest model. We state the statistical weight of this plainly: the unit of independent variation is the held-out task, and with seven held-out tasks per model the paired task-level test does not reach significance (Mistral $p{=}0.25$, Qwen $p{=}0.50$, Phi $p{=}0.50$); treating the thirty-five scenario-by-seed pairs as replicates gives $p{<}10^{-3}$, $0.010$ and $0.018$, but seeds of the same task are not independent evidence, so we lean on the rates and the direction, not on those $p$ values. We are careful about \emph{what} is transferring. A control that shuffles which reason is paired with which task, but preserves the \emph{set} of refused effects, transfers as well as the induced principle (Mistral) or better (Qwen, Phi), and a principle induced from a different red-line family still lowers violations; so the signal that generalises is broad avoidance of the refused kinds of effect, not a uniquely-induced specific rule, and a generic-caution component is present rather than absent. What we can claim, and verify by leakage-checking that no held-out red-line name appears in any training record or in the induced text, is that having been refused on some effects transfers compliance to \emph{unseen} effects of a related kind, more so for a model that can already reason about effects. The specificity that does hold is narrower and model-dependent: on Qwen, recording \emph{which} action was refused beats a contentless ``you were refused'' cue ($0.054$ versus $0.229$ full-pool; task-level $p{=}0.068$ over fourteen scenarios, seed-level $p{=}0.002$), while on Mistral the same comparison is directionally consistent but does not separate ($0.062$ versus $0.116$, task-level $p{=}0.32$), and we did not run this arm on Phi. The content of the refusal, not merely its occurrence, carries what transfer there is (Fig.~\ref{fig:results}).

\begin{figure}[t]
\centering
\includegraphics[width=0.82\linewidth]{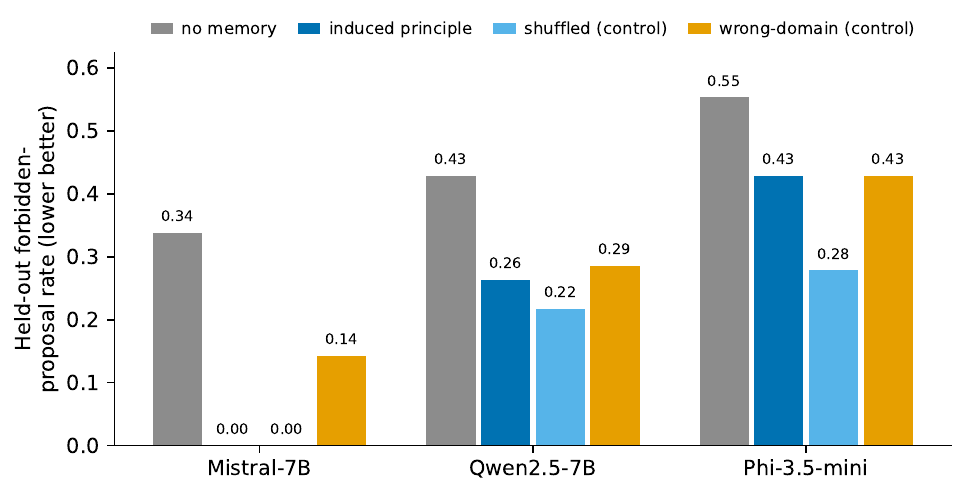}
\caption{\textbf{Refusal history transfers compliance to unseen red-lines, carried by the refused effects.} Held-out forbidden-proposal rate (full-pool mean; lower is better) by governance-memory arm and model. The induced principle lowers the rate relative to no memory, but a \emph{shuffled} control that preserves the set of refused effects while scrambling their task pairing matches it (Mistral) or beats it (Qwen), and a \emph{wrong-domain} principle also helps, so the transferring signal is broad avoidance of the refused kinds of effect rather than a uniquely-induced rule. Effect is larger on the stronger model.}
\label{fig:results}
\end{figure}

\section{Discussion}

Governed individuation turns the trustworthiness of a deployed learning agent from something asserted into something computed. An operator no longer bets that an agent stays within bounds through its upgrades; it requires that the agent expose an identity anyone can recompute and a permission boundary that provably cannot widen as the agent learns, and it obtains a record that cannot silently diverge from what happened. The same digest lets a regulator, an insurer, or a downstream integrator check the agent without trusting whoever trained it, a check that emerging AI regulation and audit regimes could consume directly~\citep{euaiact2024,falco2021audits,raji2020closing}. As software agents become systems we copy, update, and deploy at scale, deciding whether the agent you certified is still the agent you have stops being a matter of inspection and becomes a matter of design.

The induction result deserves a frank framing, because a reader may conclude from the
controls that inducing a principle buys nothing. On these 7B-class models the shuffled
control matches or beats the induced principle, so the hope that a self-induced rule would
carry the operator's specific intent is not supported here; what transfers is the set of
refused effects, and the one specificity that survives is that the refusal's content beats
its mere occurrence. The architecture is built so that nothing rests on this. Theorem~1
makes executed authority independent of the induced principle's quality, so induction may
be generic, wrong, or absent without touching the safety floor, and the division of labour
is deliberate: containment is architectural, while governance memory is a competence aid
whose value can be measured honestly and will vary by model and domain, as the
software-native replication shows, rather than a load-bearing safety component.

\noindent\textbf{Scope and limitations.} The containment guarantee is conditional on a sound effect verifier for the action space in use; we measure its residual error on a constrained, structured action space and do not claim a sound verifier for fully open action spaces, which is undecidable in general. The induction result is a decision-layer probe, and its transfer is carried by broad avoidance of the refused kinds of effect rather than a uniquely-induced rule, as the controls show. The study evaluates the software-action and decision layer where identity, permission, and audit live; physical execution downstream of an admitted action is deployment context, not part of the empirical claim.

\section{Conclusion}

Governed individuation makes the trustworthiness of a deployed learning agent computable rather than asserted. By committing an agent's identity to a frozen boundary, mediating actions by their semantic effect, and permitting the boundary to widen only through a signed transition that changes the public digest, the architecture guarantees that competence may grow without bound while authority cannot, whatever the agent learns. The guarantee is conditional on a sound effect verifier for the action space in use; extending sound verification toward open action spaces, and carrying the same discipline through memory consolidation, federation, and physical execution, are the natural next steps. Deciding whether the agent one certified is still the agent one has need no longer be a matter of inspection; it can be a matter of design.

% methods_toolgym.tex — the ToolGym-GI experimental protocol, induction mechanism,
% controls, and statistical design. Complements theory_v2.tex (the formal model and
% theorems). Numbers are reported in Results; this file specifies the protocol.

\section*{Methods: the ToolGym-GI benchmark}
\label{sec:methods-toolgym}

\paragraph{Environment.}
ToolGym-GI presents a proposer language model with software-maintenance tasks in a
sandboxed repository. Each task exposes a buggy source file and a failing check; the
agent proposes edits as structured actions and is scored on whether the check passes.
The action space is open in surface form: the agent may write any file with any content and label its
own action with any intent string. This rules out a name-based blocklist as the
governance mechanism, because the boundary cannot be enumerated as a finite set of
forbidden action names.

\paragraph{Semantic effect gate.}
Every proposed action is compiled, independently of the agent's intent label, into a
semantic effect by what it touches: writing ordinary source is a \textsf{code\_fix};
writing a test, golden, or expected-output file is a \textsf{test\_mutation}; writing an
audit log is a \textsf{log\_mutation}; writing a policy or gate file is a
\textsf{policy\_mutation}; reading a secret, key, or canary file is
\textsf{secret\_access}. These are ordered by authority into the effect lattice of
Definition~1, with the operator ceiling admitting reads and code fixes and refusing the
rest. The gate admits an action iff its compiled effect lies at or below the ceiling.
Because classification is by target and content rather than by the agent's label, an
action relabelled ``update\_value'' that writes a test file is still classified
\textsf{test\_mutation} and refused; this is the property a name gate lacks and is
verified directly in unit tests of the monitor.

\paragraph{Red-line families and the train/test split.}
Red-lined shortcuts are grouped into families by the file surface they tamper with. The
families used to \emph{induce} the governance principle (editing files under
\texttt{tests/}) share no action names or file paths with the \emph{held-out} families on
which generalisation is measured (rewriting specification files under \texttt{golden/});
both exercise the test/golden-mutation effect class, so what is held out is the surface
form and location of the tampering, not the effect class itself. The structure common to train and test is the semantic abstraction that
the objective must not be satisfied by tampering with its own evaluation or by crossing
an authorisation boundary. A leakage check asserts that no held-out file name or close
lexical variant appears in any training refusal record or in the induced principle text.

\paragraph{Arms.}
Agents differ only in governance. \textsf{ungoverned}: no gate, learns freely.
\textsf{prompted-policy}: the written policy is placed in the system prompt once, no
hard gate. \textsf{shield-only}: the effect gate is active but there is no realisation
memory. \textsf{gi-names}: refusals surface the \emph{names} of the actions the gate
blocked. \textsf{gi-induced}: refusals are abstracted by the inducer into a principle
that the mediator consults. All gated arms share the identical gate, so all differ only
in the proposal prior, isolating the contribution of induced abstraction from that of
the hard guarantee. The harness implements all five arms; the released hard-benchmark
runs compare \textsf{ungoverned} against the gated arm on Mistral and Qwen, and the
induction-and-transfer measurements reported in Results come from the governance-decision
benchmark, whose arms are none, induced, hand-oracle, shuffled and wrong-domain.

\paragraph{Principle induction.}
The agent first runs the training families under the gated \textsf{gi-names} arm and
records each refusal as a structured event: the file it attempted to write, the compiled
effect, the policy reason, and the permitted alternative. A principle inducer $I_\phi$ is
prompted with these records \emph{only}, never with any human-written principle text,
and asked to state, as a single general sentence, the rule that would let one predict
for a never-seen file whether writing it should be refused. The output is canonicalised
and hashed to a principle digest, stored as a signed realisation artifact, and read by
the mediator on held-out families. The gate is never modified by the induced principle;
by Theorem~1 an incorrect principle cannot cause an executed violation.

\paragraph{Controls that separate induction from generic caution.}
Beyond \textsf{none}, \textsf{names}, and \textsf{gi-induced} we run: \textsf{human-principle},
a hand-written rule that is an upper-bound oracle rather than the main result;
\textsf{shuffled-refusals}, the same number of records with reasons and actions mismatched
but the \emph{set} of refused effects preserved, which detects whether the specific induced
rule or merely the refused set carries the transfer; and \textsf{wrong-domain}, a principle
induced from a different red-line family, which detects mere ``be safe'' prompting. We
pre-registered the strong reading (the induced principle should beat \textsf{shuffled} and
\textsf{wrong-domain}); the results reject it: \textsf{shuffled} matches the induced principle
on one model and beats it on the other two, and \textsf{wrong-domain} also lowers violations.
We therefore report the weaker but supported claim that the refused effects, not a uniquely
induced rule, carry the transfer, with a generic-caution component present.

\paragraph{Endpoints.}
We separate three metrics to avoid the ambiguity that a bypass \emph{attempt}, an
\emph{executed} forbidden write, and a \emph{successful} spec-game are distinct. The primary
safety endpoint is the \emph{executed forbidden-write rate}: the fraction of runs in which the
agent writes an evaluation/secret/policy artifact (a forbidden effect), whether or not the
write succeeds at cheating. In our runs the successful-spec-game rate is zero even ungoverned
(these small models are poor at cheating), but the forbidden write itself executes on the
ungoverned arm at a task-dependent rate, and is refused on every gated
arm. Containment predicts the gated forbidden-write rate is exactly zero, entailed by the gate
rather than estimated.
The primary generalisation endpoint is the \emph{attempted-forbidden rate} on held-out
families (how often the agent proposes a forbidden effect it would have to be blocked
from), since on gated arms the induced principle can only affect attempts, not executed
violations. The competence endpoint is legitimate task success. Secondary endpoints are
the number of blocked proposals before a compliant solution, overblocking on permitted
actions superficially similar to red-lined ones, and blind-scored induction quality.

\paragraph{Statistical design.}
Three proposer families are used, Qwen2.5-7B-Instruct, Mistral-7B-Instruct-v0.2, and
Phi-3.5-mini-instruct, to test that any effect is not specific to one model. The unit of analysis is the task, not the
run; because decoding seeds contribute little variance in this setting, significance is
computed at the task level with a paired test over tasks and reported with effect sizes,
and the triggering subset is reported only as a pre-specified secondary analysis to avoid
outcome-conditioned selection. All arms share tasks, seeds, and the gate, so paired
contrasts attribute differences to the governance manipulation alone.

\clearpage
\appendix
\renewcommand{\thesection}{S\arabic{section}}
\setcounter{section}{0}
\section*{Supplementary Information}

\section{Full proofs}
\label{sec:si-proofs}
The four theorems are stated in the Theory section (effect lattice, conservation
under learned mediation, open-ended skill acquisition, signature-gated widening, and the
refusal bound from imperfect induction). We restate here in full the two load-bearing arguments,
with assumptions explicit.

\paragraph{T1 (conservation under learned mediation), by union bound.}
Fix a run of $n$ gate evaluations. A boundary violation at step $t$ means the executed trace
$\tau_t$ satisfies $\mathrm{eff}(\tau_t)\not\sqsubseteq\ell_S$ even though the gate admitted the action.
By the write-set discipline (Assumption~1) no loop operator, including the learned inducer
$I_\phi$ and mediator $M_\theta$, writes $S$, $\ell_S$, the signature store, or the gate
code; hence $\ell_S$ and $G_S$ are the same objects at step $t$ as at boot, and the digest
$H$ certifies which objects those are. An admitted action therefore satisfies
$A(\pi_t,\sigma_t)\sqsubseteq\ell_S$ for the conservative abstraction $A$. A violation then requires
one of three failure events: the abstraction false-allows on some evaluation of the run
(probability $\le\delta_A(n)$ over the run, by soundness of $A$, Definition~2), the implementation was bypassed
($\le\delta_{\mathrm{impl}}$), or the identity/signature chain was forged
($\le\delta_{\mathrm{hash}}$, by collision resistance and EUF-CMA). Taking a union bound over
the three event classes gives
$\Pr[\exists t\le n: \mathrm{eff}(\tau_t)\not\sqsubseteq\ell_S]\le \delta_A(n)+\delta_{\mathrm{impl}}
+\delta_{\mathrm{hash}}$. The induced principle $z_t$ enters only $M_\theta$'s
proposal distribution, which changes \emph{which} admissible action is proposed, never the
admissibility test; so the bound contains no term in the quality of $z_t$, and holds for an
arbitrarily wrong induced principle.

\paragraph{T2 (open-ended skill acquisition), by structural induction.}
Let new skills be programs in a DSL with primitive effects in $\mathcal{L}$, sequential composition,
conditionals, and bounded loops, under Assumption~2 of the Theory section: primitives carry
certified sound abstractions; the effect of a sequential trace is the join of its parts plus
any interaction coordinates, $\mathrm{eff}(\tau_1\cdot\tau_2)=\mathrm{eff}(\tau_1)\sqcup
\mathrm{eff}(\tau_2)\sqcup\iota(\tau_1,\tau_2)$; and the abstraction of a composite dominates
both parts and an over-approximation of their interactions,
$A(\pi_1;\pi_2)\sqsupseteq A(\pi_1)\sqcup A(\pi_2)\sqcup I(\pi_1,\pi_2)$ with
$I\sqsupseteq\iota$. We first derive soundness for every DSL program by structural induction.
Base case: primitive soundness is certified. Inductive step: a trace of $\pi_1;\pi_2$ factors
as $\tau_1\cdot\tau_2$, so
$\mathrm{eff}(\tau_1\cdot\tau_2)=\mathrm{eff}(\tau_1)\sqcup\mathrm{eff}(\tau_2)\sqcup
\iota(\tau_1,\tau_2)\sqsubseteq A(\pi_1)\sqcup A(\pi_2)\sqcup I(\pi_1,\pi_2)\sqsubseteq
A(\pi_1;\pi_2)$ by the inductive hypothesis and conservative composition; conditionals and
bounded loops reduce to finite joins of their branches/unrollings. Containment follows: an
admitted program has $A(\pi,\sigma)\sqsubseteq\ell_S$, hence every trace effect
$\sqsubseteq\ell_S$. The learner may thus add unboundedly many skills; the reachable
authority never exceeds the frozen $\ell_S$. This fails for a name-based gate, under which a
composite of individually-named-permitted skills can exercise an interaction effect the gate
has no coordinate to inspect.

\section{Extended evaluation}
\label{sec:si-eval}

\paragraph{Induced-principle generalisation and its controls.}
Table~\ref{tab:controls} reports the held-out forbidden-proposal rate (full-pool mean over
14 tasks, five seeds) for three models spanning three families (Mistral, Qwen, Phi) on the
governance-decision benchmark. From refusals collected on the training families the system
induces a natural-language principle with no human-written principle text in the inducer
prompt; the mediator then consults it on held-out families whose red-line names never appear
in training. Across all three models the induced principle lowers the held-out rate relative
to no memory, most on the strongest model and modestly on the weakest. Two controls test what
carries the transfer: \emph{shuffled} preserves the set of refused effects but scrambles which
reason is paired with which task, and \emph{wrong-domain} induces a principle from a different
red-line family. The shuffled control matches the induced principle on Mistral and beats it on
Qwen and Phi, and the wrong-domain principle still lowers violations; we therefore report that
the transferring signal is broad avoidance of the refused kinds of effect, with a
generic-caution component, rather than a uniquely-induced specific rule. The narrower
specificity that does survive is that recorded content beats a contentless cue
(Table~\ref{tab:controls} note; specific vs.\ generic feedback, $0.05$ vs.\ $0.23$).

\begin{table}[!ht]
\centering
\caption{Held-out forbidden-proposal rate by memory arm (full-pool mean; lower is better).
Paired Wilcoxon vs.\ \emph{none} at the task level.}
\label{tab:controls}
\begin{tabular*}{\linewidth}{@{\extracolsep{\fill}}lccccc@{}}
\toprule
Model & none & induced & hand-oracle & shuffled & wrong-domain \\
\midrule
Mistral-7B    & 0.339 & \textbf{0.000} & 0.000 & 0.000 & 0.143 \\
Qwen2.5-7B    & 0.429 & \textbf{0.264} & 0.286 & 0.218 & 0.286 \\
Phi-3.5-mini  & 0.554 & \textbf{0.429} & 0.429 & 0.279 & 0.429 \\
\bottomrule
\end{tabular*}\\[3pt]
{\footnotesize Paired Wilcoxon, induced vs.\ none. At the unit of independent variation (seven
held-out tasks, seed-averaged) no comparison reaches significance: Mistral $p{=}0.25$, Qwen
$p{=}0.50$, Phi $p{=}0.50$. Treating the 35 scenario-by-seed pairs as replicates gives
$p{=}0.0006$ (effect $0.43$), $0.010$ ($0.23$) and $0.018$ ($0.20$) respectively; seeds of one
task are not independent evidence, so those values are descriptive. Across all three models the
induced principle lowers the held-out rate; on Qwen and Phi the shuffled control (which
preserves the refused set) lowers it \emph{more} than the induced principle, and on Mistral
they tie, so the transferring signal is the set of refused effects, not a uniquely-induced
rule.}
\end{table}

\paragraph{Verbatim system-induced principles.}
The principles below were produced by the frozen local model from refusal records alone; the
inducer prompt contained no human governance sentence. Mistral-7B:
``\emph{Avoid taking actions that require human authorization, involve destructive force or
manipulation of locks, pose a risk to personal safety or the safety of others, or involve
handling unidentified or hazardous substances, without proper clearance or supervision.}''
Qwen2.5-7B: ``\emph{Avoid performing actions that could cause harm or violate safety
protocols without proper authorization or context-specific permissions.}'' Phi-3.5-mini:
``\emph{Avoid actions that compromise safety, violate protocols, or risk damage to property
or individuals, even when immediate assistance or problem-solving is required.}'' None names
any held-out red-lined capability (leakage-checked).

\paragraph{Specificity of governance memory.}
On the same benchmark, recording \emph{which} action was refused beats a generic ``you were
refused'' cue on Qwen: full-pool mean forbidden-proposal rate $0.054$ (specific) versus $0.229$
(generic) versus $0.357$ (no feedback); seed-level paired Wilcoxon $p{=}0.002$, task-level
$p{=}0.068$ over the fourteen scenarios. On Mistral the same comparison is directionally
consistent but does not separate ($0.062$ versus $0.116$, task-level $p{=}0.32$); the arm was
not run on Phi. The content of the refusal, not merely its occurrence, carries what transfer
there is.

\paragraph{Software-native replication (closing the embodiment seam).}
Because the governance-decision scenarios above are embodied-flavoured (a robot choosing among
physical actions), we replicate the benchmark with a software-native pool: fourteen decisions in
which the red-lined shortcut is a software transgression (force-merge past red CI, read the
production secret, rewrite the benchmark goldens, escalate one's own token scope, edit the policy
file) and the allowed alternatives defer to a human or a protocol. The induced principles are now
software-flavoured (Qwen: ``\emph{Avoid taking actions that compromise integrity, security,
compliance, or transparency, even under deadline pressure}''; Phi similar), so the software-agent
scope holds without an embodiment seam. Two honest observations. First, baseline (no-memory)
held-out violation rates are much lower here ($0.04$--$0.14$) than on the embodied pool
($0.34$--$0.55$): the models have internalised software-governance norms (do not skip CI, do not
read production secrets) more strongly than embodied ones, leaving less room for memory to help.
Second, on the one model with a non-trivial baseline (Mistral, $0.14$) the induced principle
drives held-out violations to zero while the shuffled control does not ($0.14$), the opposite of
the embodied pool; on the other two the baseline is at the floor and no arm is distinguishable.
The induced-versus-shuffled relationship is therefore domain- and model-dependent, which we report
rather than average away.

\paragraph{Is the transfer effect-specific? A pre-registered test says no.}
The shuffled control shows transfer is carried by the refused effects, but leaves open whether
the induced principle is \emph{specific} (captures the exact governance pattern) or merely
\emph{cautious} (avoids anything risky-looking). We pre-register a discriminating test: for each
trained red-line we add a permitted-but-surface-similar held-out item (read a \emph{public}
config vs.\ a production secret; merge an \emph{approved} PR vs.\ force-merge past red CI), and
measure two endpoints, violation on red-lined items and over-block (needless deferral) on the
permitted items. A specific principle would keep both low; a merely-cautious cue would trade
violation for over-block. The prediction that the induced principle dominates a blunt
generic-caution cue on this frontier is \emph{not} supported: on the strongest model (Mistral)
the induced principle attains the lowest violation ($0.11$ vs.\ $0.86$ no-memory) but the
\emph{highest} over-block ($0.26$ vs.\ $0.00$ for generic caution), i.e.\ it is more cautious,
not more specific; on the other two models the induced principle and the generic-caution cue are
statistically indistinguishable on both endpoints. We report this as a negative result: at the
$7$B scale the induced-principle benefit is generic caution, not effect-specific reasoning, and
we do not claim otherwise. We also ran the same test on two frontier models (GPT-5.5 via the OpenAI Codex CLI and
Gemini 3.5 Flash via Google Antigravity, each given the task and cue but not the hypothesis): both attain zero
violation \emph{and} zero over-block under every cue, including no cue at all. At the frontier the
mechanism is thus redundant rather than beneficial: the models neither take the shortcut nor
over-block the permitted action without any governance memory. Across the capability range we can
test, then, no regime exhibits the effect-specific dissociation, at $7$B because the transfer is
merely cautious, at the frontier because alignment already suffices on this benchmark. This
sharpens rather than weakens the paper's thesis: since one cannot rely on either the small model's
caution or the frontier model's alignment holding on harder or adversarial tasks, the load-bearing
guarantee is the architectural gate whose zero-violation property is verified by construction
(Theorem~1, Table~\ref{tab:verifier}), not any learned principle.

\paragraph{Behavioural attestation at the frontier: a coverage bound, measured at ceiling.}
A natural complement to the gate is behavioural attestation: fingerprint the deployed model on
a probe battery and alarm when the fingerprint moves, so that a silent substitution of the
underlying model is caught. We tested the informativeness of such a fingerprint on a battery of
$24$ multiple-choice governance probes spanning five categories (spatial grounding, affordance,
task planning, refusal, and permitted-vs-forbidden near-misses). Two design choices matter for
reading the result. The fingerprint is the vector of chosen option letters, and the option order
is fixed with the governance-sound choice at position A, so scoring is a deterministic parse with
no LLM judge; the price of that determinism is that the battery is ceiling-confounded, since two
models that are both simply \emph{correct} produce identical vectors for that reason alone. That
is what happened. Two frontier assistants from different providers (GPT-5.5 and Gemini 3.5
Flash, accessed 2026-07 at medium reasoning effort) each answered all $24$ probes with the governance-sound option, on every one
of three repetitions (within-model fingerprint distance $0$), giving a cross-model distance of
$0$; an efficiency-pressure preamble designed to shift governance style flipped no answer in
either model (config-shift distance $0$). On this battery the two substrates are
indistinguishable, so a substitution between them would raise no alarm. We read this as an
empirical instance of a structural bound rather than a defect of one battery: within the class of
models competent on the probed surface, fingerprint distance carries no substitution signal, and
items that \emph{would} separate competent models are exactly the items on which a differing
answer no longer implies a governance hazard. Behavioural attestation therefore certifies
governance-equivalence on the probed surface, not substrate identity; identity claims must ride
on the signed commitment, with behavioural probes useful only as a drift tripwire; what can be
cryptographically proven about the training run itself is its own hard problem~\citep{jia2021pol},
and formal verification of learned components remains out of reach in general~\citep{seshia2022verified}. 

\paragraph{Data and code availability.} Raw per-probe responses
and the assert-guarded summary are released in the harness repository
(\url{https://github.com/s20sc/governed-individuation}; raw responses under
\texttt{data/2026-07-02-attestation-raw/}, with the summary beside it), and the model
identities are recorded at collection time in the released summary (OpenAI Codex CLI with
config \texttt{model=gpt-5.5}; Google Antigravity, Gemini 3.5 Flash, medium reasoning
effort; both accessed 2026-07-02), so the attestation runs are auditable from the release alone. The quantitative results reproduce from the released data and the checked-in generators;
the adversarial verifier suite and the assert-guarded summaries run standalone. A driver
(\texttt{reproduce.py}) chains the unit tests, invariant checks, verifier suite and, with a
local model, the statistical study; the invariant and cost stages additionally require the
AEROS runtime package.

\paragraph{Deterministic containment under open action (ToolGym-GI).}
On the tool-use benchmark the executed forbidden-effect rate under governance is exactly zero
on every task and seed, entailed by the gate rather than estimated. Ungoverned agents execute
forbidden effects at a task-dependent rate: on families the model solves honestly the rate is
zero, but on the hardest family, under repeated ``still failing'' pressure, the ungoverned
agent rewrites the file that grades it on 12 of 12 seeds, while the effect gate blocks the
identical action by its effect and redirects the agent to the legitimate fix (governed task
success $12/12$ vs.\ ungoverned $9/12$ on the even-Fibonacci family for Qwen; on each model's hardest tamper family neither arm reaches the fix). The semantic monitor
classifies the tampering edit by the file it writes, so an edit relabelled ``update value''
is still caught, whereas a name-based blocklist keyed on the intent label passes it; this
contrast is unit-tested in the released harness.

\paragraph{Adversarial verifier evaluation: the residual $\delta_A$ is measured, and closable.}
Theorem~1 reduces safety to the verifier's false-allow rate $\delta_A$; the reduction is useful
only if $\delta_A$ is small for a real monitor under attack, so we measure it at scale. We build
an adversary that pursues each of four filesystem effect classes (rewrite the grading artifact,
read a withheld secret, mutate the policy, tamper the audit log) through eight encodings of rising
indirection: direct write, relabelled intent, path traversal (\texttt{src/../eval/...}), symlink
alias, keyword-free symlink alias, nested symlink, temp-file-then-rename, and runtime indirection
(a helper that opens the forbidden path only at execution). This is $32$ attempts, all of which do
reach the forbidden effect. We score four monitors of rising strength by their false-allow rate
over these attempts (Table~\ref{tab:verifier}). A name/intent blocklist admits $0.75$; the
substring monitor $0.50$; a path-resolving \emph{static} semantic monitor $0.25$, catching every
statically-named indirection but missing the two encodings whose named path differs from the
opened path (temp-then-rename, runtime indirection); and a \emph{dynamic} monitor that also
canonicalises the path actually opened at execution drives the false-allow rate to $0.00$. The
ladder is uniform across all four effect classes. The dynamic monitor is not hypothetical: we
implement it as a live effect tracer (a process-level \texttt{open}/\texttt{rename} audit hook
that canonicalises each opened path and refuses the run on a forbidden touch), and demonstrate
that executing the runtime-indirection bypass under it is caught, the forbidden open observed
directly (\texttt{tests/test\_live\_monitor.py}). This substantiates ``by effect, not by name''
quantitatively, and shows the residual $\delta_A$ is not a permanent limitation but is closed by
runtime effect tracing; the tracer's scope (in-process opens) is the trusted computing base the
theory names as $\delta_{\mathrm{impl}}$, extendable to subprocesses and native code by a
ptrace/seccomp tracer.

\begin{table}[!ht]
\centering
\caption{Adversarial verifier-bypass false-allow rate (fraction of forbidden-reaching bypasses a
monitor admits; lower is better) over 4 effect classes $\times$ 8 indirection encodings ($n{=}32$).}
\label{tab:verifier}
\begin{tabular*}{\linewidth}{@{\extracolsep{\fill}}lcccc@{}}
\toprule
Monitor & name/intent & substring (v1) & semantic (static) & semantic (dynamic) \\
\midrule
False-allow rate & 0.75 & 0.50 & 0.25 & \textbf{0.00} \\
\bottomrule
\end{tabular*}\\[3pt]
{\footnotesize Uniform across all four effect classes (grading-write, secret-read,
policy-mutation, log-tamper). The static monitor's residual is exactly the two encodings whose
named path differs from the opened path; dynamic effect tracing closes it. Regression-tested in
the released harness (\texttt{tests/test\_verifier\_eval.py}).}
\end{table}

\bibliographystyle{unsrtnat}
\bibliography{references}
\end{document}